\documentclass[10pt,twocolumn,letterpaper]{article}
\usepackage{algorithm}
\usepackage{algorithmic}
\usepackage{cvpr}              

\definecolor{cvprblue}{rgb}{0.21,0.49,0.74}
\usepackage[pagebackref,breaklinks,colorlinks,allcolors=cvprblue]{hyperref}
\usepackage{multirow}
\usepackage[capitalize,noabbrev]{cleveref}
\usepackage{xcolor}
\usepackage{amsthm}
\newtheorem{lemma}{Lemma}
\newtheorem{corollary}{Corollary}
\newtheorem{proposition}{Proposition}
\DeclareMathOperator{\CE}{CE}
\DeclareMathOperator{\PPL}{PPL}
\DeclareMathOperator{\softmax}{softmax}
\def\paperID{21}
\def\confName{CVPR}
\def\confYear{2026}
\title{A KL Lens on Quantization: Fast, Forward-Only Sensitivity for Mixed-Precision SSM--Transformer Models}
\usepackage{float}
\usepackage{listings}
\author{Jason Kong\\
{\tt\small jckong@ucsd.edu}
\and
Nilesh Prasad Pandey\\
{\tt\small nppandey@ucsd.edu}
\and
Flavio Ponzina\\
{\tt\small fponzina@sdsu.edu}
\and
Tajana Rosing\\
{\tt\small tajana@ucsd.edu}
}

\begin{document}
\maketitle


\begin{abstract}
Deploying Large Language Models (LLMs) on edge devices faces severe computational and memory constraints, limiting real-time processing and on-device intelligence. Hybrid architectures combining Structured State Space Models (SSMs) with transformer-based LLMs offer a balance of efficiency and performance. Aggressive quantization can drastically cut model size and speed up inference, but its uneven effects on different components require careful management. In this work, we propose a lightweight, backpropagation-free, surrogate-based sensitivity analysis framework to identify hybrid SSM--Transformer components most susceptible to quantization-induced degradation. Relying solely on forward-pass metrics, our method avoids expensive gradient computations and retraining, making it suitable for situations where access to in-domain data is limited due to proprietary restrictions or privacy constraints. We also provide a formal analysis showing that the Kullback--Leibler (KL) divergence metric better captures quantization sensitivity for Language modeling tasks than widely adopted alternatives such as mean squared error (MSE) and signal-to-quantization-noise ratio (SQNR). Through extensive experiments on SSM and hybrid architectures, our ablation studies confirm that KL-based rankings align with observed performance drops and outperform alternative metrics. This framework enables the practical deployment of advanced hybrid models on resource-constrained edge devices with minimal accuracy loss. We further validate our approach with real-world on-device profiling on Intel Lunar Lake hardware, demonstrating that KL-guided mixed-precision achieves near-FP16 perplexity with model sizes and throughput competitive with Uniform INT4 on both CPU and GPU execution modes. Code is available at \url{https://github.com/jasonkongie/kl-ssm-quant}.

\end{abstract}

\section{Introduction}
Large Language Models (LLMs) based on the Transformer architecture~\cite{vaswani2017attention} have achieved remarkable success across NLP tasks, but their deployment on edge devices is constrained by high memory and computation demands. Recently, State Space Models (SSMs)~\cite{gu2023mamba} have emerged as efficient alternatives for long sequence modeling, offering $\mathcal{O}(n)$ sequence processing and favorable memory scaling. SSM-based architectures can maintain strong accuracy on sequence tasks while being more hardware-friendly, which makes them attractive for resource-constrained platforms. In parallel, \textit{hybrid models}~\cite{dong2024hymba,glorioso2024zamba} combining SSM and Transformer components have been proposed to capture the best of both worlds. These hybrids achieve competitive performance with improved efficiency, exemplified by models like Zamba~\cite{glorioso2024zamba} and Hymba~\cite{dong2024hymba} that interleave SSM and attention modules.

Despite their efficiency, even SSM-based hybrids at scale can contain hundreds of millions to billions of parameters, remaining impractical for direct deployment on mobile and IoT devices. Quantization~\cite{nagel2021white,krishnamoorthi2018quantizing,ashkboos2024quarot} is a proven approach to compress models by reducing the numerical precision of weights~\cite{nagel2020adaptive,dong2019hawq} and activations~\cite{xiao2023smoothquant,pandey2023softmax}. 4-bit uniform quantization of Transformers has become standard in industry for inference speed-ups, and recent studies demonstrate that even lower-bit weight quantization has become feasible for LLMs given proper handling of outlier features~\cite{dettmers2022llmint8,xiao2023smoothquant,yao2022zeroquant,pandey2026qmc}. However, applying uniform quantization to all parts of a hybrid SSM model can lead to disproportionate accuracy degradation in certain layers. Recent work~\cite{pierro2024mamba} notes that SSMs and hybrid models exhibit different quantization behavior than Transformers: for instance, the state transition in selective SSMs has highly sensitive internal activations, and SSM outputs contain large magnitude outliers not seen in Transformer outputs. These differences call for a more \emph{nuanced}, layer-wise approach to quantization.

In this paper, we demonstrate that quantizing SSM--Transformer models for edge deployment is both necessary and tractable. While prior mixed-precision techniques~\cite{pandey2023practical,dong2019hawq} have shown success in CNNs and Transformers by using gradient-based methods or signal-to-quantization-noise ratio (SQNR) for sensitivity analysis, these approaches fall short in the context of language modeling. In particular, we find that SQNR, though effective for convolutional architectures, fails to capture the true sensitivity of components in language models. Our key contribution is a novel, gradient-free sensitivity analysis method tailored for SSM--Transformer architectures. It operates entirely via forward-pass signals and reveals which layers truly require higher precision. This insight enables a post-training quantization pipeline that achieves substantial compression with minimal accuracy degradation, addressing a gap where existing mixed-precision strategies break down.

The remainder of the paper is organized as follows. 
\Cref{sec:related} reviews prior work on quantization for LLMs, SSMs, and hybrid architectures. 
\Cref{sec:method} presents our forward-pass sensitivity analysis, including per-layer metrics, mixed-precision assignment, and Kendall's~$\tau$ ranking. 
\Cref{sec:kl_vs_sqnr} formally shows that KL divergence is a tighter proxy for perplexity than SQNR. 
\Cref{sec:hymba-sensitivity} reports ablations on the Hymba hybrid model at the layer and parameter level. 
\Cref{sec:experiments} presents experiments and results, including layer-level correlation analysis across Mamba and hybrid architectures, confirming that $\mathrm{KL}_{\mathrm{student}\to\mathrm{teacher}}$ outperforms SQNR and other metrics. 
Finally, \Cref{sec:ondevice} demonstrates end-to-end deployment on Intel Lunar Lake hardware, achieving near-FP16 perplexity with latency and throughput competitive with Uniform INT4 on CPU and GPU.

\section{Related Work}
\label{sec:related}

Prior work on efficient sequence modeling can be viewed along two largely orthogonal axes: \textbf{model architecture} and \textbf{quantization strategy}. Along the first axis, recent systems include Transformer-based large language models (LLMs), state-space models (SSMs), and hybrids that combine the two. Along the second axis, quantization ranges from \textbf{homogeneous} schemes that use a single bit-width throughout the model to \textbf{mixed-precision} schemes that assign different precisions to different components based on sensitivity. Together, these two axes provide a useful way to understand how prior work trades off efficiency, accuracy, and hardware friendliness.

Early quantization work focused mainly on \emph{homogeneous quantization}, especially for Transformer models. Uniform low-bit quantization is useful because it maps well to integer hardware and simplifies deployment. Early work in vision showed that carefully calibrated 8-bit quantization can preserve accuracy \citep{krishnamoorthi2018quantizing}, and later work adapted similar ideas to LLMs using outlier-aware calibration and improved rounding methods \citep{nagel2021white}. However, large Transformers show strong layerwise variation in sensitivity, with some activations and weights exhibiting strong outlier behavior. As a result, forcing all layers to use the same precision can lead to noticeable degradation unless extra techniques such as per-layer clipping are used \citep{yang2023efficient}. This motivates going beyond one-size-fits-all precision assignment.

Mixed-precision quantization addresses this issue by allocating bit-widths non-uniformly across the network. The main idea is to keep sensitive parts at higher precision while quantizing more robust parts more aggressively. HAWQ is an early example, using Hessian-based sensitivity estimates to assign layerwise precisions and showing strong compression--accuracy trade-offs in CNNs \citep{dong2019hawq}. Similar ideas were later extended to LLMs. For example, \citet{pandey2023practical} propose a post-training method that uses a calibration set to allocate 4/8-bit weights across LLM layers, outperforming uniform int8 baselines on GLUE~\cite{wang2018glue} and SQuAD~\cite{rajpurkar2016squad}. Related work also explores low-bit fine-tuning \citep{dettmers2023qlora}, adaptive post-training rounding \citep{nagel2020adaptive}, and hardware-aware bit allocation \citep{wang2019haq}. Overall, these works show that mixed precision often gives a better accuracy--efficiency trade-off than homogeneous quantization for LLMs.

Quantization research has also started to extend beyond Transformers to \emph{state-space models}. SSMs offer a different efficiency profile, with linear-time sequence processing and constant-memory recurrent state updates, making them attractive for long-context and edge settings. Mamba shows that modern SSMs can achieve competitive sequence modeling performance at lower compute cost than Transformers on several long-range tasks \citep{gu2023mamba}. Quantization studies for these models are still limited, but existing results already suggest different sensitivity patterns from Transformers. \citet{pierro2024mamba} provide one of the first systematic studies of post-training quantization for Mamba-style recurrent LLMs, showing that uniform int8 quantization can retain strong performance when sensitive state-transition components are treated carefully, in some cases with quantization-aware training. At the same time, pushing all components to int4 leads to much larger degradation, suggesting that homogeneous quantization is mainly viable for SSMs at moderate precision.

These observations naturally motivate \emph{hybrid architectures} that combine SSM and Transformer components, where both architectural heterogeneity and precision heterogeneity become important. Recent models such as Hymba and Zamba interleave recurrent state-space computation with attention and feed-forward blocks to capture the efficiency advantages of SSMs while retaining Transformer flexibility \citep{dong2024hymba, glorioso2024zamba}. Such models are especially well suited to mixed-precision deployment because their submodules play different computational roles. In practice, recurrent state-transition matrices may need higher precision to preserve long-range dynamics, while feed-forward and attention-related components can often be quantized more aggressively. This makes hybrid models a natural setting for sensitivity-guided precision allocation.

Overall, prior work points to two consistent themes. First, quantization sensitivity is highly non-uniform, making homogeneous precision increasingly suboptimal as models become larger and more diverse. Second, this non-uniformity becomes even stronger in hybrid SSM--Transformer architectures, where recurrent and attention-based components have different numerical requirements. Our work builds on this intersection and studies how sensitivity-guided precision assignment can better exploit the structure of hybrid sequence models.

\begin{figure*}[t]
  \centering
  \begin{subfigure}{0.48\textwidth}
    \centering
    \includegraphics[width=\linewidth]{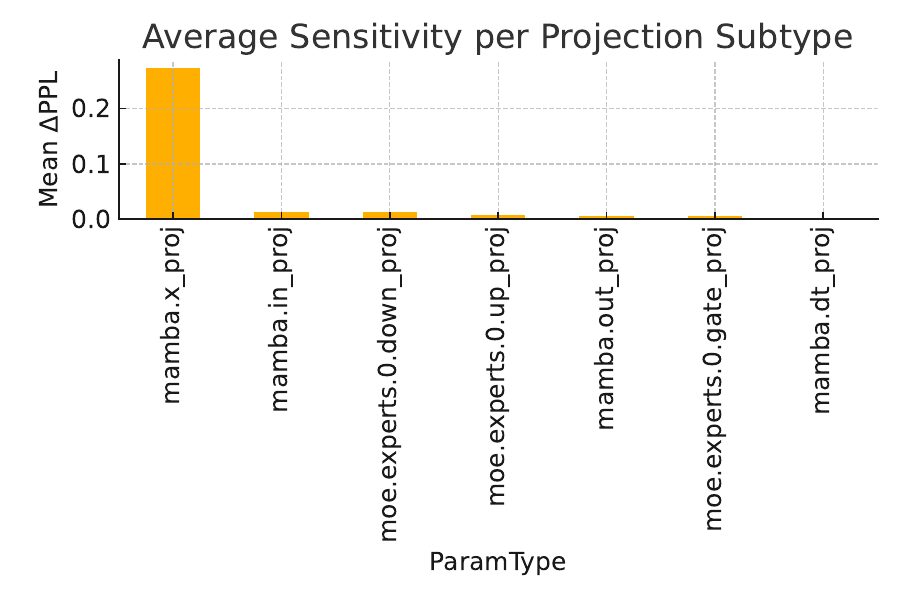}
  \end{subfigure}
  \hfill
  \begin{subfigure}{0.48\textwidth}
    \centering
    \includegraphics[width=\linewidth]{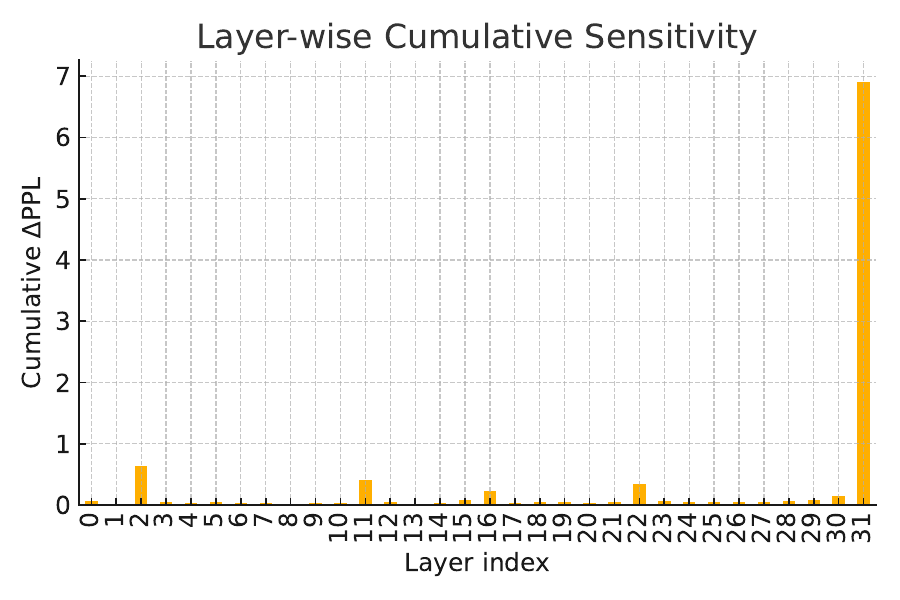}
  \end{subfigure}
  \caption{Quantization-sensitivity characteristics of Hymba.
           \textbf{(a)} Average ablation sensitivity
           $\overline{\Delta}\!\operatorname{PPL}$ for each projection
           subtype, with \texttt{mamba.x\_proj} dominating the profile.
           \textbf{(b)} Layer-wise cumulative sensitivity
           $\Sigma\Delta\!\operatorname{PPL}$, showing a sharp peak at
           block 31.}
  \label{fig:hymba-sensitivity-two}
\end{figure*}
\section{Methodology}
\label{sec:method}
We propose a lightweight sensitivity analysis method that evaluates each layer or module of a model for quantization robustness \emph{without requiring backpropagation}. In contrast, our method uses forward-pass computations only, making it scalable to very large models and avoiding any need for retraining or fine-tuning during the analysis phase.

Our core idea is to measure how much quantization of a single layer affects the model's output or performance metric. Given a trained model, we perform a series of modified inference passes over a representative dataset:

\begin{algorithm}
\caption{\textsc{KL-PPL Correlation via Kendall $\tau$}}
\label{alg:kl_ppl_corr}
\begin{algorithmic}[1]
\REQUIRE Model $\mathcal{M}$ with $L$ layers; quantization function $\textsc{Quantize}(\cdot)$; evaluation dataset $\mathcal{D}$
\ENSURE Kendall $\tau$ and $p$-value between KL and PPL rankings
\STATE Initialize lists: $\text{KL} \gets []$, $\text{PPL} \gets []$
\FOR{$\ell = 1$ \TO $L$}
  \STATE $\text{Teacher} \gets \mathcal{M}$
  \STATE $\text{Student} \gets \textsc{Quantize}(\mathcal{M}, \ell)$
  \STATE $\text{logits}_{T}, \text{PPL}_{T} \gets \textsc{Evaluate}(\text{Teacher}, \mathcal{D})$
  \STATE $\text{logits}_{S}, \text{PPL}_{S} \gets \textsc{Evaluate}(\text{Student}, \mathcal{D})$
  \STATE $\text{KL}_{\ell} \gets \textsc{ComputeKL}(\text{logits}_{T}, \text{logits}_{S})$
  \STATE Append $\text{KL}_{\ell}$ to $\text{KL}$; append $\text{PPL}_{S}$ to $\text{PPL}$
\ENDFOR
\STATE $\pi_{\mathrm{KL}} \gets \textsc{Argsort}(\text{KL}, \downarrow)$
\STATE $\pi_{\mathrm{PPL}} \gets \textsc{Argsort}(\text{PPL}, \downarrow)$
\STATE $(\tau, p) \gets \textsc{KendallTau}(\pi_{\mathrm{KL}}, \pi_{\mathrm{PPL}})$
\RETURN $(\tau, p)$
\end{algorithmic}
\end{algorithm}

\begin{table}[H]
\centering
\caption{Sort direction for each metric ($\uparrow$: ascending, $\downarrow$: descending).}
\label{tab:sorting_directions}
\begin{tabular}{lc}
\toprule
\textbf{Metric} & \textbf{Order} \\ \midrule
Perplexity                       & $\downarrow$ \\
SQNR (dB)                        & $\uparrow$  \\
KL$_{\text{teacher}\to\text{student}}$  & $\downarrow$ \\
KL$_{\text{student}\to\text{teacher}}$  & $\downarrow$ \\
$\Delta$ Cross Entropy           & $\downarrow$ \\ \bottomrule
\end{tabular}
\end{table}

This requires no gradients and is easily parallelizable.

\subsection{Efficient Mixed-Precision Assignment}
Given sensitivity scores, we assign a low bit-width to most layers and a higher bit-width to the top-$k$ sensitive ones. This yields significant compression with minimal accuracy loss.
Given model output and input logits, we perform a comprehensive analysis to discover a correlation between perplexity and metrics discussed in \cref{sec:kl_vs_sqnr}.

\subsection{Overview of Kendall's $\tau$ Ranking}
\label{sec:kendall}
Kendall's $\tau$~\cite{kendall1938} is a non-parametric statistical measure that quantifies the ordinal correlation between two rankings. Specifically, given two rankings $R_1$ and $R_2$ of $n$ elements, Kendall's $\tau$ is defined as:
\[
\tau(R_1, R_2) = \frac{|\mathcal{C}| - |\mathcal{D}|}{\frac{1}{2} n (n - 1)},
\]
where $\mathcal{C}$ and $\mathcal{D}$ represent the sets of concordantly and discordantly ranked pairs, respectively. The coefficient ranges from $-1$ (complete disagreement) to $+1$ (complete agreement), with $0$ indicating no correlation. Unlike Pearson's correlation, Kendall's $\tau$ relies solely on rank order, making it particularly suitable for assessing ranking-based sensitivity metrics.

\section{Sensitivity Scoring in SSMs}

\subsection{Analytical Comparison of \texorpdfstring{$D_{\mathrm{KL}}$}{KL} and SQNR as Sensitivity Metrics for Mixed-Precision Language Models}
\label{sec:kl_vs_sqnr}

Recent works~\cite{pandey2023practical,yang2023efficient} establish the \emph{signal-to-quantization-noise ratio} (SQNR) as a reliable sensitivity measure for CNNs and transformer-based models. SQNR quantifies the distortion introduced by quantizing model logits compared to the original (teacher) logits, expressed in decibels (dB):
\[
\mathrm{SQNR} = 10 \log_{10}\frac{\mathbb{E}[\,\ell_{\mathrm{orig}}^2\,]}{\mathbb{E}[(\ell_{\mathrm{orig}}-\ell_{\mathrm{quant}})^2]}.
\]
Higher SQNR values signify that quantization retains the original activations more accurately.

However, for autoregressive language models, we demonstrate that the \emph{Kullback--Leibler divergence}~\cite{kullback1951} ($D_{\mathrm{KL}}$) generalizes better than SQNR. KL divergence measures distributional shifts due to quantization in both directions to account for potential asymmetries:
\[
\mathrm{KL}(P \| Q) = \mathbb{E}_{P}\left[\log P - \log Q\right],
\]
where $P = \mathrm{softmax}(\ell_{\mathrm{orig}})$ and $Q = \mathrm{softmax}(\ell_{\mathrm{quant}})$. Higher KL values indicate greater divergence between the original and quantized distributions.

While SQNR captures signal fidelity at the logit level, it does not reliably correspond to downstream performance in language tasks. In contrast, KL divergence provides a tighter link to the change in perplexity ($\PPL$), making it more suitable for sensitivity scoring in autoregressive settings.

\paragraph{Notation.}
Let $x$ be a context and $y\!\in\!V$ a token.
Throughout,
\[
\begin{aligned}
  p(y\mid x)            &:\; \text{full-precision \emph{teacher} model},\\
  q_{\theta}(y\mid x)   &:\; \text{quantized model}.
\end{aligned}
\]
Expectations $\mathbb{E}_{p}[\cdot]$ are taken over $(x,y)\!\sim\!p$.
As all expectations $\mathbb{E}_{p}[\,\cdot\,]$ are taken with respect to the
teacher, the results are \emph{relative} to that uncompressed model.\footnote{If we use the test-set distribution for $p$, the formulas stay the same.
The only change is that $H(p)$ now depends on that specific test data.}

\vspace{2pt}
\begin{lemma}[Cross-entropy split]\label{lem:ce_decomp}
\begin{align}
\CE(q_{\theta},p)
  &= -\mathbb{E}_{p}\![\log q_{\theta}(y\mid x)]    \\[2pt]
  &= H(p) + D_{\mathrm{KL}}(p\!\parallel\!q_{\theta}).
\end{align}
\end{lemma}

\begin{lemma}[PPL factorization]\label{lem:ppl_fact}
\begin{align}
\PPL(q_{\theta})
  &= \exp\!\bigl(\CE(q_{\theta},p)\bigr) \\[2pt]
  &= \PPL(p)\,
     \exp\!\bigl(D_{\mathrm{KL}}(p\!\parallel\!q_{\theta})\bigr),
\end{align}
where $\PPL(p)=\exp(H(p))$.
\end{lemma}

\begin{corollary}[Teacher-relative PPL bound]\label{cor:ppl_bound}
If $D_{\mathrm{KL}}(p\!\parallel\!q_{\theta})\le\varepsilon$ then
\[
  \PPL(q_{\theta}) \le \PPL(p)\,e^{\varepsilon}.
\]
\end{corollary}

\begin{proposition}[SQNR Is Not Monotonic in $\PPL$]
\label{prop:sqnr_non_monotone}
There exist logit pairs $\bigl(z^{(1)},\hat z^{(1)}\bigr)$ and $\bigl(z^{(2)},\hat z^{(2)}\bigr)$ such that
\begin{align}
\mathrm{SQNR}\!\bigl(z^{(1)},\hat z^{(1)}\bigr)
&< \mathrm{SQNR}\!\bigl(z^{(2)},\hat z^{(2)}\bigr), \\[4pt]
\PPL^{(1)} &< \PPL^{(2)}.
\end{align}
\end{proposition}

\begin{proof}
\emph{Constant-shift example:}
Let $\hat z = z + c\mathbf{1}$ with any $c\!\neq\!0$.
Since $\softmax(z)=\softmax(\hat z)$, the models share the same $\PPL$, but
\begin{align}
\mathrm{SQNR}(z,\hat z)
  &= 10 \log_{10}\!
     \frac{\lVert z \rVert_{2}^{2}}
          {\lVert z - \hat z \rVert_{2}^{2}} \\[4pt]
  &= 10 \log_{10}\!
     \frac{\lVert z \rVert_{2}^{2}}
          {c^{2}\,\lVert \mathbf{1} \rVert_{2}^{2}} \\[4pt]
  &\xrightarrow{|c|\;\to\;\infty} -\infty.
\end{align}
So SQNR can be made arbitrarily small while $\PPL$ is unchanged.
\end{proof}

\subsection{Metric Asymmetry}

We observe a strong asymmetry between the two KL directions used to compare
student and teacher distributions. While
$KL_{\text{student}\rightarrow\text{teacher}}$ aligns well with the
perplexity-based sensitivity ranking, the reverse direction
$KL_{\text{teacher}\rightarrow\text{student}}$ shows a weak negative
correlation ($\tau \approx -0.14$).

This difference arises because the two directions penalize different types of
errors. $KL_{\text{student}\rightarrow\text{teacher}}$ strongly penalizes
probability mass that the student assigns to tokens the teacher considers
unlikely—precisely the type of distortion introduced by quantization.
In contrast, $KL_{\text{teacher}\rightarrow\text{student}}$ focuses on whether
the student underestimates the teacher’s high-probability tokens, which tends
to remain relatively stable under quantization. As a result, the reverse KL is
less sensitive to the perturbations that drive perplexity degradation.

\paragraph{Practical Implication.}
As shown above, the \emph{KL divergence gap}
$\Delta D_{\mathrm{KL}}(p \,\|\, q_{\theta})$ provides an upper bound proxy for perplexity, making it the preferred
layer-sensitivity metric for language modelling tasks.
Hence, in mixed-precision quantization of language models, precision
allocation should be driven by $\Delta D_{\mathrm{KL}}$.

\label{sec:experiments}

\begin{figure*}[t]
  \centering
  \begin{subfigure}{0.5\textwidth}
    \centering
    \includegraphics[width=\linewidth]{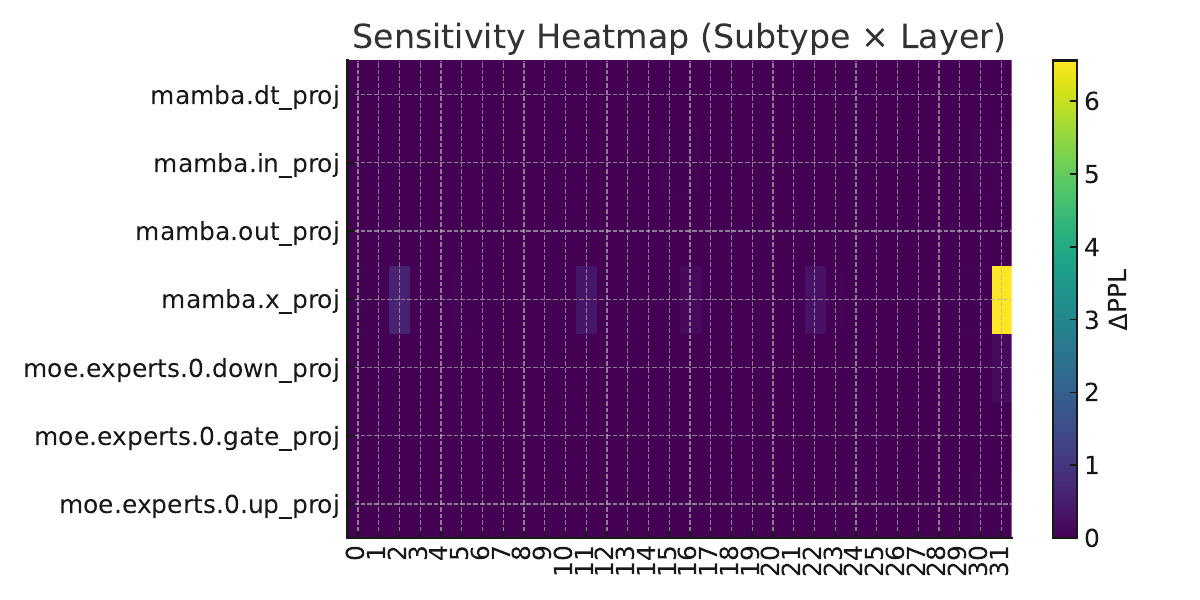}
    \caption{Sensitivity heat-map ($\Delta\!\operatorname{Perplexity}$;
             rows: projection subtype, columns: layer).  Late-stage
             \texttt{x\_proj} peaks coincide with moderate MoE projection
             peaks, underscoring their coupled importance for accurate
             mixed-precision deployment.}
    \label{fig:heatmap}
  \end{subfigure}
  \hfill
  \begin{subfigure}{0.48\textwidth}
    \centering
    \includegraphics[width=\linewidth]{hymba_layerwise.pdf}
    \caption{Layer-wise cumulative sensitivity $\Sigma\Delta\!\operatorname{PPL}$,
             showing a sharp peak at block 31.}
    \label{fig:layer}
  \end{subfigure}
  \caption{Quantization-sensitivity characteristics of Hymba.}
  \label{fig:hymba-sensitivity-combined}
\end{figure*}

\section{Ablations Studies on Hymba Hybrid Model}
\subsection{Layer- and Parameter-Level Quantization Sensitivity in Hymba}
\label{sec:hymba-sensitivity}

We present a systematic, tensor-level ablation study of the 32-layer
Hymba model~\cite{dong2024hymba}.  Although similar analyses
are common for CNNs and pure-Transformer models, we are not aware of a
prior, publicly documented examination of quantization sensitivity in
hybrid SSM--Transformer architectures; our findings therefore offer a
reference point for subsequent research on this emerging model class.
All experiments are conducted under a \textbf{uniform 4-bit
(INT4) quantization scheme}, a setting that is now widely supported on
commercial mobile and server-class accelerators.

\paragraph{Subtype imbalance.}
Among the seven projection sub-modules, the cross-projection
\texttt{mamba.x\_proj} exhibits the highest average sensitivity
$(\overline{\Delta}=0.27)$, exceeding all other subtypes by more
than an order of magnitude $(\overline{\Delta}\le 0.014)$.
Conversely, \texttt{mamba.dt\_proj} is nearly insensitive
$(\overline{\Delta}=3.6\times10^{-4})$, suggesting that extremely
low-precision representations are feasible (\cref{fig:heatmap}).

\paragraph{Localized layer hotspots.}
Aggregating $\Delta\!\operatorname{Perplexity}$ across sub-modules reveals
that block 31 alone accounts for over 70\% of the aggregate
sensitivity budget $(\Sigma\Delta=6.9)$, with smaller peaks in
blocks 2, 11, 16, and 22 (\cref{fig:layer}).  Most remaining
blocks tolerate aggressive quantization.

\paragraph{Coupled \texttt{x\_proj}--MoE behaviour.}
The heat-map in \cref{fig:heatmap} shows that late-stage
\texttt{x\_proj} sensitivities co-occur with moderate MoE
\texttt{up/down} sensitivities in the same block, indicating that
quantization noise in mixing pathways can be amplified by expert
routing.

\subsection{Ranking Correlation Protocol}

We evaluate how effectively each sensitivity metric predicts the quantization impact at a per-layer granularity, defined by the increase in perplexity ($\PPL\uparrow$) when that layer alone is quantized to \texttt{int8}. Following prior practice~\cite{pandey2023practical}, we establish the ground-truth ranking based on $\Delta \PPL$ and measure the alignment with proxy metrics---namely, SQNR, KL divergences (teacher-to-student and student-to-teacher), and $\Delta$CE---using Kendall's $\tau$ coefficient (see \cref{sec:kendall}). Metrics with higher $\tau$ values better reflect the true sensitivity of layers.

\section{Experiments \& Results}

\subsection{Layer-Level Correlation Analysis}
\cref{tab:kendall_tau} summarizes the Kendall's $\tau$ results across several hybrid SSM models of varying scales. Notably, the student-to-teacher KL divergence consistently achieves the highest correlation, averaging a Kendall's $\tau$ of \textbf{0.79}, surpassing the widely used SQNR metric (average $\tau = 0.76$). Statistical significance tests (paired one-sided t-tests) further validate that KL divergence significantly outperforms SQNR ($p<10^{-6}$). This improved correlation is explained by KL divergence operating directly in probability space, aligning closely with the downstream metric of interest (perplexity), whereas SQNR measures only logit-level fidelity and can overlook distributional shifts (\cref{prop:sqnr_non_monotone}).

\begin{table*}[h]
  \caption{Kendall's $\tau$ values for each metric across models.}
  \label{tab:kendall_tau}
  \centering
  \small
  \resizebox{\linewidth}{!}{%
  \begin{tabular}{lccccccc}
    \toprule
    Metric & Mamba-130M & Mamba-380M & Mamba-1.4B & Mamba2-130M & Hymba & Zamba & Avg. \\
    \midrule
    SQNR (dB) & $0.6967$ & $0.7314$ & $0.7941$ & $0.4898$ & $0.7124$ & $0.8419$ & $0.7111$ \\
    KL$_{\mathrm{teacher}\!\to\!\mathrm{student}}$ & $-0.1980$ & $-0.2316$ & $-0.2607$ & $-0.0272$ & $-0.1617$ & $0.1142$ & $-0.1275$ \\
    KL$_{\mathrm{student}\!\to\!\mathrm{teacher}}$ & $\mathbf{0.8419}$ & $\mathbf{0.7936}$ & $\mathbf{0.8327}$ & $\mathbf{0.8078}$ & $\mathbf{0.6646}$ & $\mathbf{0.8060}$ & $\mathbf{0.7911}$ \\
    $\Delta$CE & $-0.1400$ & $-0.1662$ & $-0.0679$ & $0.2364$ & $-0.1260$ & $-0.1234$ & $-0.0645$ \\
    \bottomrule
  \end{tabular}
  }
\end{table*}

\begin{table*}[h]
  \caption{p-values for each metric across models.}
  \label{tab:p_values}
  \centering
  \small
  \begin{tabular}{lcccccc}
    \toprule
    Metric & Mamba-130M & Mamba-380M & Mamba-1.4B & Mamba2-130M & Hymba & Zamba \\
    \midrule
    SQNR (dB) & $4.97\times10^{-24}$ & $1.5\times10^{-51}$ & $1.91\times10^{-60}$ & $6.869\times10^{-7}$ & $5.87\times10^{-57}$ & $4.57\times10^{-52}$ \\
    KL$_{\mathrm{teacher}\!\to\!\mathrm{student}}$ & $4.06\times10^{-3}$ & $1.729\times10^{-6}$ & $7.30\times10^{-8}$ & $7.827\times10^{-1}$ & $3.07\times10^{-4}$ & $3.95\times10^{-2}$ \\
    KL$_{\mathrm{student}\!\to\!\mathrm{teacher}}$ & $2.51\times10^{-34}$ & $2.29\times10^{-60}$ & $2.84\times10^{-66}$ & $2.637\times10^{-16}$ & $8.31\times10^{-50}$ & $7.1\times10^{-48}$ \\
    $\Delta$CE & $4.21\times10^{-2}$ & $5.97\times10^{-4}$ & $1.61\times10^{-1}$ & $1.656\times10^{-2}$ & $4.92\times10^{-3}$ & $2.61\times10^{-2}$ \\
    \bottomrule
  \end{tabular}
\end{table*}

\begin{figure}[t]
    \centering
    \includegraphics[width=0.95\linewidth]{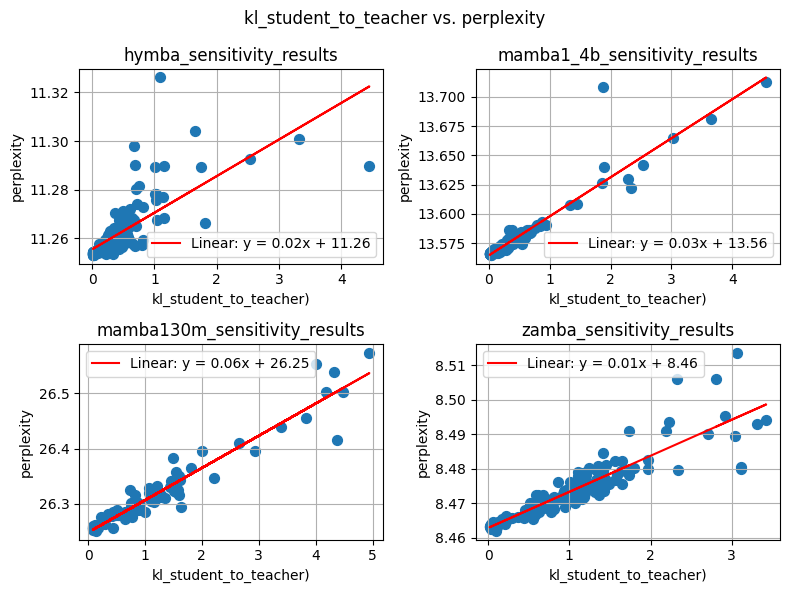}
    \caption{KL-Divergence vs.\ Perplexity correlation.}
    \label{fig:kl-ppl-corr}
\end{figure}

\subsection{Ranking-Correlation Algorithm}
\label{sec:pseudocode}

We detail the Kendall's $\tau$ ranking-correlation algorithm used to measure the alignment between our sensitivity metrics and the observed perplexity ($\PPL$) change.

\paragraph{Ranking Protocol.} We first sort layers by the observed $\Delta \PPL$ as the ground truth ranking. We then compare this ranking against rankings obtained from each sensitivity metric.





We validate this approach end-to-end with real-world profiling on
\textbf{Intel Lunar Lake} hardware~\cite{intel_lunar_lake}, covering
CPU and GPU execution across two Mamba model scales, and show that
KL-guided mixed-precision achieves near-FP16/INT8 perplexity while
matching or exceeding INT4 throughput, a combination that homogeneous
quantization cannot provide.



\section{On-Device Profiling of Mixed-Precision Quantization}
\label{sec:ondevice}

\subsection{Experimental Setup}
\label{sec:ondevice-setup}

We profile models end-to-end on an Intel \textbf{Lunar Lake}
platform~\cite{intel_lunar_lake}, which integrates a CPU, dedicated GPU,
and Neural Processing Unit (NPU) on the same die, making it
representative of next-generation AI-capable edge clients.
Models are converted to OpenVINO IR via the XAMBA
framework~\cite{xamba2024} and benchmarked with the OpenVINO
\texttt{benchmark\_app} (latency mode, 100 iterations).
We evaluate \textbf{Mamba-130M} and \textbf{Mamba-1.4B}~\cite{gu2023mamba}
on CPU, and \textbf{Mamba2-130M}~\cite{gu2023mamba} on GPU, comparing
against \textbf{FP16}, \textbf{Uniform INT8}, and \textbf{Uniform INT4}
baselines.
We report model size (MB), latency (ms), throughput (FPS), and
perplexity (WikiText-2).

\begin{table*}[t]
  \centering
  \caption{On-device results: best KL-MP point vs.\ baselines on Intel
           Lunar Lake. $\uparrow$~higher is better; $\downarrow$~lower
           is better. ``--'' * The 130M model exhibited substantially higher perplexity under uniform INT4 quantization. \dag FPS and latency are not reported, as INT8 execution is not supported for the certain relevant operations on Intel Lunar Lake}
  \label{tab:ondevice_summary}
  \small
  \resizebox{0.75\linewidth}{!}{%
  \begin{tabular}{llccccc}
    \toprule
    Model & Config & PPL $\downarrow$ & Size & FPS $\uparrow$ & Latency $\downarrow$ \\
    \midrule
    \multirow{4}{*}{\shortstack[l]{Mamba-130M\\(CPU)}}
      & FP16                      & 21.59 & 493\,MB & 22.6 & 44\,ms     \\
      & INT8                      & 25.32 & 126\,MB & 35.5 & 28\,ms     \\
      & INT4*                      & 3e35    &  84\,MB & 39.7 & 25\,ms     \\
      & \textbf{KL-MP p05 (ours)} & \textbf{21.61} & \textbf{136\,MB} & \textbf{35.9} & \textbf{28\,ms} \\
    \midrule
    \multirow{4}{*}{\shortstack[l]{Mamba-1.4B\\(CPU)}}
      & FP16                      & 11.22 & 5.2\,GB  & 2.6 & 384\,ms    \\
      & INT8                      & 11.25 & 1.3\,GB  & 5.1 & 196\,ms    \\
      & INT4                      & 60.55 & 723\,MB  & 5.3 & 190\,ms    \\
      & \textbf{KL-MP p05}        & \textbf{11.22} & \textbf{1.4\,GB} & \textbf{5.6} & \textbf{178\,ms} \\
    \midrule
    \multirow{3}{*}{\shortstack[l]{Mamba2-130M\\(GPU)}}
      & FP16                      & 46.45 &  81\,MB  & 0.02 & 60006\,ms \\
      & INT8\dag                      & 53.03 & 125\,MB  & --   & --         \\
      & \textbf{KL-MP p02}        & \textbf{46.46} & \textbf{45\,MB} & \textbf{0.29} & \textbf{3417\,ms} \\
    \bottomrule
  \end{tabular}
  }
\end{table*}
\paragraph{Mixed-precision configurations.}
Each configuration \textbf{$p_{k}$} ($k=1,\ldots,10$) corresponds to a
distinct KL threshold $\kappa_k$, above which layers are retained at
FP16 and below which they are compressed to INT4 (CPU) or INT8 (GPU).
Formally, given per-layer KL sensitivity scores $\{s_\ell\}_{\ell=1}^{L}$,
configuration \textbf{p$k$} assigns:
\begin{equation}
  b_\ell =
  \begin{cases}
    \text{FP16}      & \text{if } s_\ell \ge \kappa_k, \\
    \text{INT4/INT8} & \text{otherwise,}
  \end{cases}
  \label{eq:threshold}
\end{equation}
where $\kappa_1 > \kappa_2 > \cdots > \kappa_{10}$.
\textbf{p01} applies the most conservative threshold, quantizing only
the least sensitive layers; \textbf{p10} applies the most aggressive
threshold, approaching near-uniform quantization.
The configurations thus trace a smooth compression--accuracy curve
controlled entirely by the KL sensitivity scores.

\subsection{CPU Profiling: Mamba-130M and Mamba-1.4B}
\label{sec:cpu-profiling}

\Cref{fig:cpu_ppl_mem,fig:cpu_lat_mem,fig:cpu_fps_mem} show perplexity,
latency, and throughput plotted against model size for both CPU models.
KL-guided quantization progressively compresses Mamba-130M from
493\,MB to 84\,MB ($5.9\times$) and Mamba-1.4B from 5.2\,GB to
723\,MB ($7.2\times$), matching Uniform INT4 at the most aggressive
setting.
Perplexity remains near the FP16 baseline through \textbf{p06} for
Mamba-130M and \textbf{p09} for Mamba-1.4B; beyond these thresholds,
a sharp sensitivity cliff emerges---consistent with the ablation
finding that a single high-sensitivity layer can dominate the entire
quantization error budget (\cref{sec:hymba-sensitivity}).
\textbf{p05} represents the best operating point for both models,
achieving INT8-level or better throughput (35.9 vs.\ 35.5\,FPS;
5.6 vs.\ 5.1\,FPS) with near-FP16 perplexity, and delivering the
lowest observed latency for Mamba-1.4B (178\,ms)---below both Uniform
INT8 (196\,ms) and INT4 (190\,ms).

\begin{figure}[t]
  \centering
  \includegraphics[width=\linewidth]{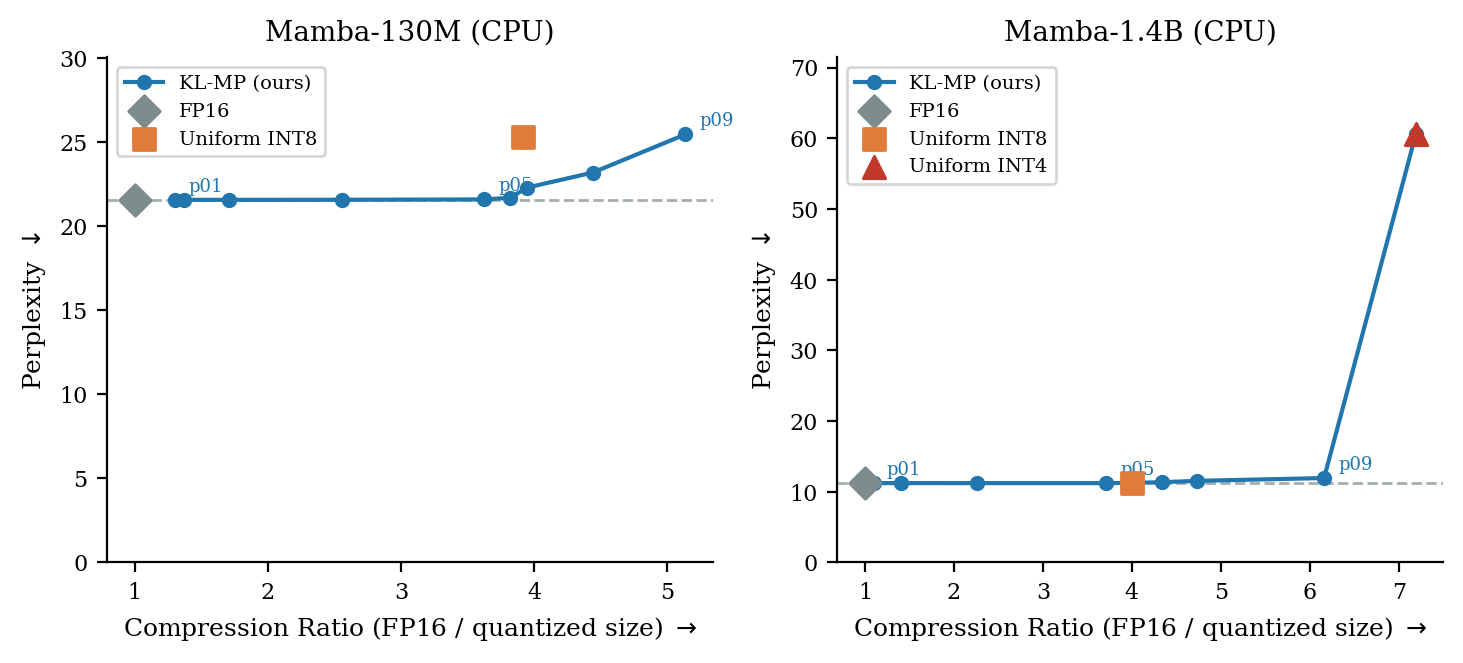}
  \caption{Perplexity vs.\ model size (CPU, Intel Lunar Lake). The
           x-axis is inverted so more compressed models appear to the
           right. KL-MP maintains near-FP16 accuracy across a wide
           compression range. The sharp cliff at \textbf{p10} for
           Mamba-130M (diverged) and \textbf{p10} for Mamba-1.4B
           ($\mathrm{PPL}=60.6$) reflects the quantization of a single
           high-sensitivity layer identified in our ablation
           (\cref{sec:hymba-sensitivity}); Uniform INT4 crosses this
           boundary by construction.}
  \label{fig:cpu_ppl_mem}
\end{figure}

\begin{figure}[t]
  \centering
  \includegraphics[width=\linewidth]{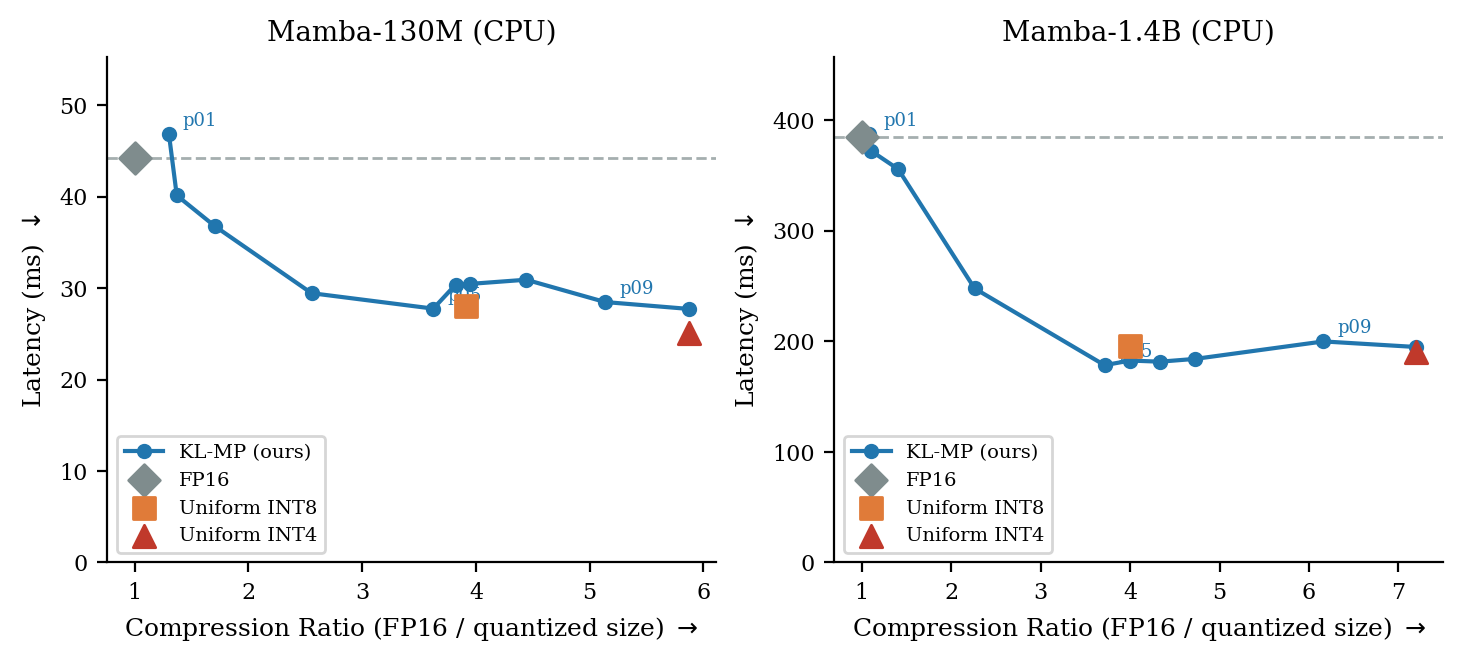}
  \caption{Inference latency vs.\ model size (CPU, Intel Lunar Lake).
           KL-MP \textbf{p05} achieves the lowest latency for
           Mamba-1.4B (178\,ms), undercutting both Uniform INT8 and
           INT4, demonstrating that sensitivity-guided precision
           allocation reduces not only memory but also compute overhead.}
  \label{fig:cpu_lat_mem}
\end{figure}

\begin{figure}[t]
  \centering
  \includegraphics[width=\linewidth]{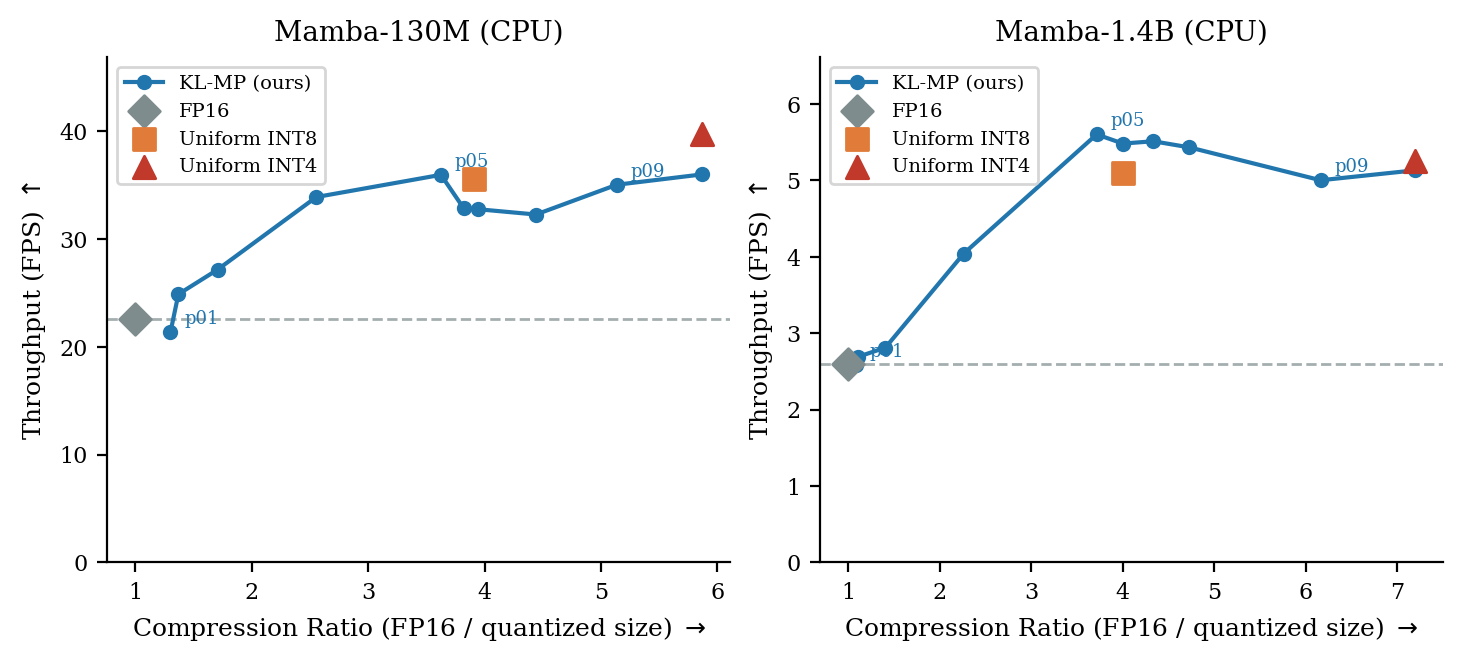}
  \caption{Throughput vs.\ model size (CPU, Intel Lunar Lake). KL-MP
           \textbf{p05} matches or exceeds INT4 throughput for both
           models while preserving near-FP16 perplexity, occupying a
           region of the accuracy--efficiency space inaccessible to
           homogeneous quantization.}
  \label{fig:cpu_fps_mem}
\end{figure}

\subsection{GPU Profiling: Mamba2-130M}
\label{sec:gpu-profiling}

\Cref{fig:gpu_combined} reports GPU results as two complementary views:
perplexity vs.\ throughput (left) and perplexity vs.\ size.

\paragraph{Why GPU memory savings are negligible.}
Unlike the CPU models, KL-MP configurations cluster around
$\approx$45\,MB regardless of $k$, versus 81\,MB for the FP16 baseline.
This is a consequence of two pipeline constraints: (1) the XAMBA
exporter applies \texttt{compress\_to\_fp16=True} by default, so the
FP16 baseline is already half the FP32 size before any further
quantization; and (2) the GPU pipeline applies INT8 rather than INT4,
yielding only a $2\times$ reduction over FP16 on the quantized
projection layers, while the non-quantizable embedding table dominates
the file size throughout.
Consequently, efficiency gains on GPU manifest as \textbf{latency and
throughput improvements} from faster INT8 arithmetic---not memory
reduction.

\paragraph{Latency and accuracy.}
The FP16 baseline incurs $60{,}006$\,ms due to iGPU memory-transfer
overhead, while KL-MP reduces this by up to $\mathbf{17.6\times}$
(\textbf{p02}: 3{,}417\,ms, 0.29\,FPS).
Perplexity holds at $46.45$--$46.48$ for \textbf{p01}--\textbf{p09},
matching the FP16 baseline; only \textbf{p10} crosses the precision
cliff to $53.03$, identical to Uniform INT8. This confirms that the
KL top-$k$ selection correctly identifies the critical
layer boundary.

\begin{figure}[t]
  \centering
  \includegraphics[width=\linewidth]{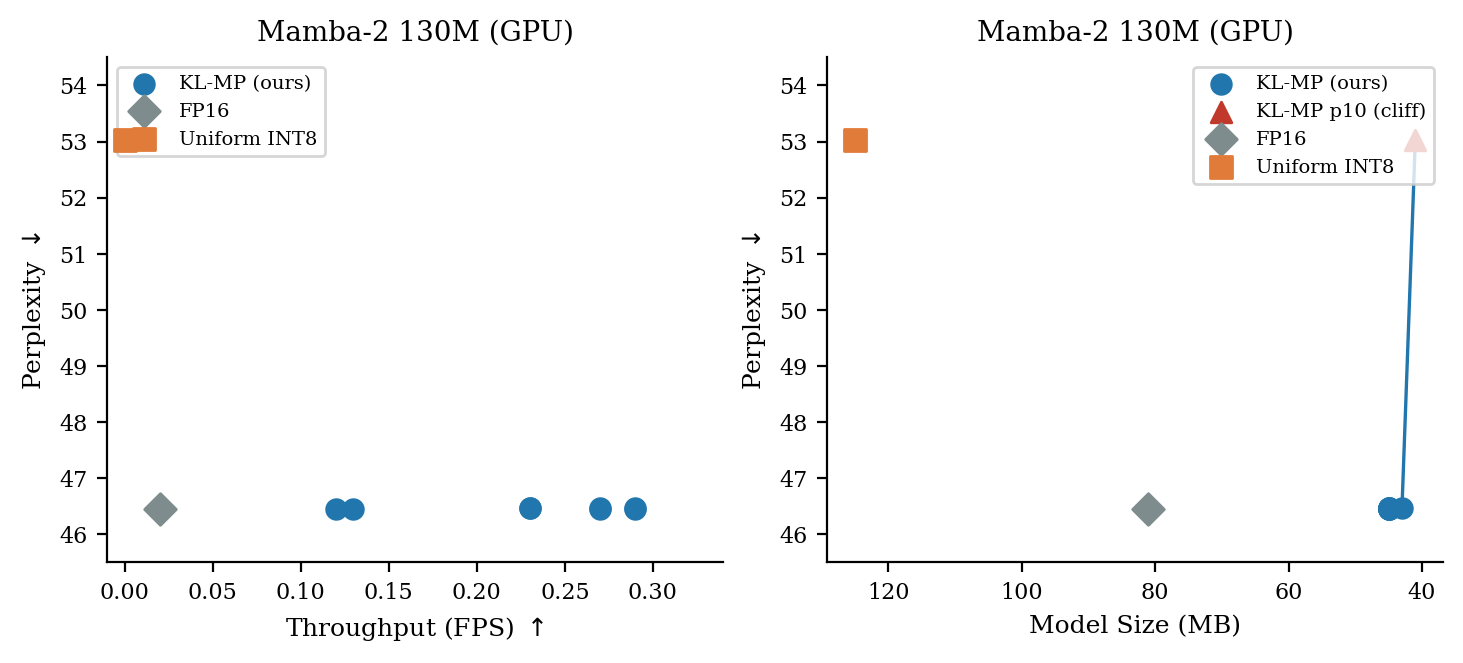}
  \caption{GPU profiling of Mamba2-130M on Intel Lunar Lake.
           \textbf{(Left)} Perplexity vs.\ throughput: KL-MP
           configurations \textbf{p01}--\textbf{p09} achieve up to
           $14.5\times$ throughput improvement over FP16 with
           negligible perplexity cost.
           \textbf{(Right)} Perplexity vs.\ model size: file sizes
           cluster near 45\,MB due to FP16 export compression and
           INT8-only quantization; only \textbf{p10} (red triangle)
           crosses the precision cliff, collapsing to Uniform INT8
           perplexity.}
  \label{fig:gpu_combined}
\end{figure}

\section{Conclusion}
\label{sec:conclusion}

In this work, we introduced a practical and scalable sensitivity analysis framework tailored specifically for hybrid SSM-Transformer models, enabling efficient mixed-precision quantization without requiring backpropagation or retraining. By systematically evaluating quantization sensitivity through forward-pass metrics, we identified critical layers where quantization-induced errors significantly impact model performance. Our results on language modeling tasks demonstrate that selectively maintaining higher precision in these sensitive layers effectively mitigates accuracy degradation, achieving up to $5.9$–$7.2\times$ reduction in model size with minimal perplexity loss. Future directions include extending our analysis framework to broader hybrid architectures and exploring dynamic quantization strategies that adapt precision allocation during inference, further enhancing model efficiency on resource-constrained edge devices. Real-world profiling on Intel Lunar Lake confirms these findings end-to-end: KL-guided mixed-precision reduces Mamba-1.4B from 5.2\,GB to 1.4\,GB with no measurable perplexity loss, and cuts Mamba2-130M GPU latency by up to 17.6$\times$ over the FP16 baseline.

\section*{Acknowledgments}
The authors gratefully acknowledge Intel Corporation for donating the hardware used in this research and for the XAMBA framework~\cite{xamba2024}, which informed our approach to efficient SSM deployment. This work has been funded in part by NSF, with award numbers \#2112665, \#2112167, \#2003279, \#2120019, \#2211386, \#2052809, \#1911095 and in part by PRISM and CoCoSys, centers in JUMP 2.0, an SRC program sponsored by DARPA.

\newpage
{
    \small
    \bibliographystyle{ieeenat_fullname}
    \bibliography{main}
}

\newpage
\lstset{
  basicstyle=\ttfamily\footnotesize,
  columns=fullflexible,
  keepspaces=true,
  breaklines=true,
  breakatwhitespace=true,
  frame=single,
  backgroundcolor=\color{gray!8},
  commentstyle=\color{gray},
  keywordstyle=\color{cvprblue},
  aboveskip=6pt,
  belowskip=6pt
}
\appendix
\section{On-Device Profiling: Reproducibility Guide}
\label{app:profiling}

All profiling scripts target Intel Lunar Lake via OpenVINO.
The pipeline runs in three stages.

\subsection*{Step 1 — Convert and Quantize}

\begin{lstlisting}[language=bash]
python convert.py             # base models
python quantize_uniform.py    # INT8 / INT4
python quantize_mixed.py      # KL M-P (CPU)
python quantize_mixed_gpu.py  # KL M-P (GPU)
\end{lstlisting}

\noindent\textbf{\texttt{convert.py}} exports each Hugging Face Mamba and
Mamba-2 checkpoint to an OpenVINO FP16 IR (\texttt{.xml}/\texttt{.bin})
using the OpenVINO Model Conversion API.
The resulting FP16 baselines are written to \path{ov_models/} and serve as
the starting point for all quantization scripts.

\noindent\textbf{\texttt{quantize\_uniform.py}} produces two fixed-precision
endpoints for each model: a fully INT8-symmetric model and a fully
INT4-symmetric model, both quantized per-channel with no sensitivity
guidance.
These bracket the Pareto curve — uniform INT8 gives the highest quality at
larger model size, and uniform INT4 gives the smallest size at lowest
quality — and serve as baselines for comparing the mixed-precision
configurations in between.
SSM \texttt{conv1d} layers (OpenVINO type \texttt{Convolution}) are excluded
from quantization in all scripts; for Mamba-2 the XAMBA CumBA MatMul node
is additionally excluded because it has no INT8/INT4 GPU kernel on Intel
Lunar Lake iGPU.

\noindent\textbf{\texttt{quantize\_mixed.py}} implements our KL-guided
mixed-precision method for the CPU pipeline, targeting Mamba-130M and
Mamba-1.4B.
It merges the per-layer 4-bit and 8-bit KL sensitivity data into a single
list, sorting entries from least to most sensitive.
Each layer therefore appears twice — once for its 4-bit cost and once for its
8-bit cost — with last-wins semantics resolving ties.
Ten evenly-spaced cutoff points are drawn through this merged list.
At each cutoff, a two-pass NNCF compression is applied to the FP16 baseline:
Pass~1 compresses the designated layers to INT4\_SYM; Pass~2 compresses a
separate set to INT8\_SYM; remaining layers stay FP16.

\noindent\textbf{\texttt{quantize\_mixed\_gpu.py}} implements the GPU
pipeline, which is restricted to INT8/FP16 mixed precision because Intel
Lunar Lake iGPU does not expose INT4 weight kernels.
It loads only the 8-bit sensitivity data, sorts layers from least to most
sensitive, and applies a single-pass INT8\_SYM compression that ignores the
sensitive tail.
Mamba-2 is capped at eight points (\texttt{point01}--\texttt{point08})
because the surrounding INT8 context at higher compression levels causes
oneDNN to fail its primitive descriptor for the XAMBA CumBA MatMul node,
even when that node is explicitly excluded from quantization.

The sensitivity metric is controlled by a constant inside each script.
All output filenames automatically inherit the selected tag
(e.g., \path{mamba-130m-hf_kl_point05.xml}):

\begin{lstlisting}[language=Python]
SENSITIVITY_METRIC = "kl_student_to_teacher"
METRIC_TAG = "kl"
\end{lstlisting}

The ten configurations \textbf{p01--p10} are produced by splitting the
sensitivity-ranked layer list into equal segments.
Configuration \textbf{p01} quantizes only the least sensitive layers,
while \textbf{p10} approaches near-uniform quantization.

\subsection*{Step 2 — Benchmark}

\begin{lstlisting}[language=bash]
python benchmark.py      # CPU latency + throughput
python benchmark_gpu.py  # GPU latency + throughput
\end{lstlisting}

Each model is evaluated using the OpenVINO
\texttt{benchmark\_app} tool:

\begin{lstlisting}[language=bash]
benchmark_app -m <model>.xml \
  -d CPU -hint latency \
  -t 60 -niter 50 \
  --inference_only TRUE
\end{lstlisting}

The flags \texttt{-t 60} and \texttt{-niter 50} ensure that each run
executes for at least 60 seconds and at least 50 iterations,
whichever takes longer.

Results are written to \path{log/benchmark_log/} as
per-model text logs and a summary CSV file:
\path{latency_throughput_{device}_{tag}_report.csv}.

\begin{lstlisting}[language=bash]
python eval_perplexity.py      # CPU, WikiText-2 PPL
python eval_perplexity_gpu.py  # GPU, WikiText-2 PPL
\end{lstlisting}

Perplexity is computed on the WikiText-2 test set.
Fake quantization is applied to replicate each configuration's
weight precision during evaluation.

\begin{lstlisting}[language=Python]
ds = load_dataset(
    "wikitext", "wikitext-2-raw-v1", split="test")
text = "\n\n".join(
    t for t in ds["text"] if t.strip())
\end{lstlisting}

Results are stored in
\path{perplexity_results_{tag}.json} with the structure

\begin{lstlisting}
{model: {point: perplexity}}
\end{lstlisting}

This format allows direct comparison across sensitivity
metrics and model scales.

\end{document}